\documentclass[final]{siamltex}

\usepackage[parfill]{parskip}    
\usepackage{amsmath,amssymb,latexsym,enumerate,mathrsfs}
\usepackage{hyperref}
\usepackage{array}
\usepackage{times}
\usepackage[]{mdframed}
\usepackage{epstopdf}
\usepackage{subfig}
\usepackage{graphicx}
\usepackage{hyperref} 
\usepackage{multicol}
\usepackage{framed}
\usepackage{moreverb}
\usepackage{marvosym}
\usepackage{color,soul}
\definecolor{lightblue}{rgb}{.90,.95,1}
\sethlcolor{yellow}

\newtheorem{remark}{Remark}

\title{A naive aggregation algorithm for improving generalization in a class of learning problems}

\author{Getachew K. Befekadu}

\begin{document}
\maketitle

\renewcommand{\thefootnote}{\arabic{footnote}}

\begin{abstract}
In this brief paper, we present a naive aggregation algorithm for a typical learning problem with expert advice setting, in which the task of improving generalization, i.e., model validation, is embedded in the learning process as a sequential decision-making problem. In particular, we consider a class of learning problem of point estimations for modeling high-dimensional nonlinear functions, where a group of experts update their parameter estimates using the discrete-time version of gradient systems, with small additive noise term, guided by the corresponding subsample datasets obtained from the original dataset. Here, our main objective is to provide conditions under which such an algorithm will sequentially determine a set of mixing distribution strategies used for aggregating the experts' estimates that ultimately leading to an optimal parameter estimate, i.e., as a consensus solution for all experts, which is better than any individual expert's estimate in terms of improved generalization or learning performances. Finally, as part of this work, we present some numerical results for a typical case of nonlinear regression problem.
\end{abstract}

\section{Introduction} \label{S1}
The main objective of this brief paper is to present an aggregation algorithm for a typical learning problem with expert advice setting, in which the task of improving generalization, i.e., model validation, is embedded in the learning process as a sequential decision-making problem with dynamic allocation scenarios. In particular, the learning framework that we propose here can be viewed as an extension for enhancing the learning performance in a typical empirical risk minimization-based learning problem of point estimations for modeling of high-dimensional nonlinear functions, when we are dealing with large datasets as an instance of {\it divide-and-conquer} paradigm. To be more specific, we have a group of experts that update their parameter estimates using a discrete-time version of gradient systems with small additive noise term, guided by a set of subsample datasets obtained from the original dataset by means of bootstrapping with/without replacement or other related resampling-based techniques. Here, our interest is to provide conditions under which such an algorithm will determine a set of mixing distribution strategies used for aggregating the experts' parameter estimates that ultimately leading to an optimal parameter estimate, i.e., as a consensus solution for all experts, which is better than any individual expert's estimate in terms of improved generalization or learning performances.

This brief paper is organized as follows. In Section~\ref{S2}, we present our main results, where we provide an aggregation algorithm that can be viewed as an extension for enhancing the learning performance and improving generalization in a typical learning problem with the help of a group of experts. In Section~\ref{S3}, we present numerical results for a typical case of nonlinear regression problem, and Section~\ref{S4} contains some concluding remarks.

\section{Main results} \label{S2} In this section, we present a learning framework with expert advice setting, that can be viewed as an extension for enhancing the learning performances in a typical empirical risk minimization-based learning problem, where the task of improving generalization is embedded in the learning process as a sequential decision-making problem with dynamic allocation scenarios.

In what follows, the learning framework that we propose consists of $(K+1)$ subsample datasets of size $m$ (where $m$ is much less than the total dataset points $d$) that are generated from a given original dataset $Z^d = \bigl\{ (x_i, y_i)\bigr\}_{i=1}^d$ by means of bootstrapping with/without replacement or other related resampling-based techniques, i.e.,
\begin{align}
\hat{Z}^{(k)} = \bigl\{ (\hat{x}_i^{(k)}, \hat{y}_i^{(k)})\bigr\}_{i=1}^m, \label{Eq2.1}
\end{align}
where $(\hat{x}_i^{(k)}, \hat{y}_i^{(k)}) \in Z^d$ with $i \in \{1,\,2, \ldots, d\}$ and $k=1, 2, \ldots, K+1$. Moreover, we use the datasets $\hat{Z}^{(k)}$, $k=1,\,2, \ldots, K$, for parameter estimation, i.e., model training purpose, corresponding to each of the $K$ experts, (i.e., a group of $K$ experts, numbered $k = 1, 2, \dots, K$), while the last dataset $\hat{Z}^{(K+1)} = \bigl\{ (\hat{x}_i^{(k)}, \hat{y}_i^{(k)})\bigr\}_{i=1}^m$ will be used for improving generalization in the learning process (i.e., how well each expert's estimate performs as part of the model validation process). Here, each expert is tasked to search for a parameter $\theta \in \Gamma$, from a finite-dimensional parameter space $\mathbb{R}^p$, such that the function $h_{\theta}(x) \in \mathcal{H}$, i.e., from a given class of hypothesis function space $\mathcal{H}$, describes best the corresponding dataset used during the model training and validation processes. 

In terms of mathematical optimization construct, searching for an optimal parameter $\theta^{\ast} \in \Gamma \subset \mathbb{R}^p$ can be associated with a {\it steady-state solution} to the following gradient system, whose {\it time-evolution} is guided by the corresponding subsampled dataset $\hat{Z}^{(k)}$
\begin{align}
 \dot{\theta}^{(k)}(t) = - \nabla J_k(\theta^{(k)}(t),\hat{Z}^{(k)}), \quad \theta^{(k)}(0) = \theta_0, \quad k=1,2, \ldots, K, \label{Eq2.2}
\end{align}
with $J_k(\theta^{(k)}, \hat{Z}^{(k)}) = \frac{1}{m} \sum\nolimits_{i=1}^m {\ell}\bigl(h_{\theta^{(k)}}(\hat{x}_i^{(k)}), \hat{y}_i^{(k)}\bigr)$, where $\ell$ is a suitable loss function that quantifies the lack-of-fit between the model (e.g., see \cite{r1} for general discussions on learning via dynamical systems). Here, we specifically allow each expert to update its parameter estimate using a discrete-time version of the above related differential equations with small additive noise term, i.e.,
\begin{align}
 d\Theta_t^{(k)} = - \nabla J_k(\Theta_t^{(k)},\hat{Z}^{(k)})dt + \left(\epsilon/\sqrt{\log(t + 2)}\right) I_p dW_t^{(k)},& \quad \Theta_0^{(k)} = \theta_0^{(k)}, \notag \\
  & k =1,2, \ldots, K, \label{Eq2.3}
\end{align}
where $\epsilon > 0$ is very small positive number, $I_p$ is a $p \times p$ identity matrix, and $W_t^{(k)}$ is a $p$-dimensional standard Wiener process. 

Moreover, we remark that if $\nabla J_k(\theta,\hat{Z}^{(k)})$, for each $k \in \{=1,2, \ldots, K\}$, is uniformly Lipschitz and further satisfies the following growth condition
\begin{align}
 \bigl\vert \nabla J_k(\theta,\hat{Z}^{(k)}) \bigr\vert^2 \le L_{\rm Lip} \bigl(1 + \vert \theta \vert^2 \bigr), \quad \forall\theta \in \Gamma \subset \mathbb{R}^p, \label{Eq2.4}
\end{align}
for some constant $L_{\rm Lip} > 0$. Then, the distribution of $\Theta_t^{(k)}$ converges to the limit of the Gibbs densities proportional to $\exp \left (-J_k(\theta,\hat{Z}^{(k)})/T\right)$  as the {\it absolute temperature} $T$ tends to zero, i.e.,
\begin{align}
 T = \left(\epsilon/\sqrt{\log(t + 2)}\right)^2 \to 0 \quad \text{as} \quad t \to \infty, \label{Eq2.5}
\end{align}
which is expected to be concentrated at the global minimum $\theta^{\ast} \in \Gamma \subset \mathbb{R}^p$ of $J_k(\theta,\hat{Z}^{(k)})$.

Note that, for an equidistant discretization time $\delta=\tau_{n+1} - \tau_n = T/N$, $n=0,1,2, \ldots, N-1$, with $0=\tau_0 < \tau_1< \ldots < \tau_n<\ldots<\tau_N=T$, of the time interval $[0,T]$, the {\it Euler-Maruyama} approximation for the continuous-time stochastic processes $\Theta^{(k)} =\bigl\{ \Theta_t^{(k)}, \, 0 \le t \le T \bigr\}$, $k=1,2, \ldots, K$, satisfying the following iterative scheme     
\begin{align}
 \Theta_{n+1}^{(k)} = \Theta_{n}^{(k)} - \delta\nabla J_k(\Theta_n^{(k)},\hat{Z}^{(k)}) &+ \left(\epsilon/\sqrt{\log(\tau_{n+1} + 2)}\right)I_p \Delta W_n^{(k)},  \quad \Theta_0^{(k)} = \theta_0^{(k)}, \notag\\
 & \quad \quad\quad \quad n=0,1,\ldots, N-1,   \label{Eq2.6}
\end{align}
where we have used the notation $\Theta_{n}^{(k)}=\Theta_{\tau_n}^{(k)}$ and the increments $\Delta W_n^{(k)} = (W_{n+1}^{(k)}-W_{n}^{(k)})$ are independent Gaussian random variables with mean $\mathbb{E} (\Delta W_n^{(k)})=0$ and variance $\operatorname{Var}(\Delta W_n^{(k)}) = \delta I_p$ (e.g., see \cite{r2} or \cite{r3}).

Then, we formalize our decision-making paradigm with dynamic allocation scenarios that can be viewed as an extension for enhancing the learning performance as well as improving generalization in the learning framework. Here, at each iteration time step $n = 0, 1, 2, \ldots, N-1$, we decide on a mixing distribution $\pi_n =\left(\pi_n(1), \pi_n(2), \ldots, \pi_n(K)\right)$ over strategies, with $\pi_n(k) > 0$, for all $k \in \{1,2,\ldots, K\}$, and sum to one, i.e., $\sum\nolimits_{k=1}^K \pi_n(k) = 1$, that will be used for dynamically apportioning the weighting coefficients with respect to the current parameter estimates. Moreover, for each expert's current parameter estimate, we associate a risk measure $r_n(k) \in [0,1]$, based on an exponential function, which is determined by the current estimate $\Theta_{n}^{(k)}$ together with the validating dataset $\hat{Z}^{(K+1)}$, i.e.,
\begin{align}
 r_n(k) = 1 - \exp\left(-\frac{\gamma}{m} \sum\nolimits_{i=1}^m {\ell}\left(h_{\Theta_{n}^{(k)}}(\hat{x}_i^{(K+1)}), \hat{y}_i^{(K+1)}\right) \right), ~~ n &= 0, 1, 2, \ldots, N-1, \notag \\
                                                                                                                                      & \gamma > 0. \label{Eq2.7}
\end{align}

\begin{remark} \label{R1}
Note that the risk measure $r_n(k)$, at each iteration time step $n = 0,1, 2, \ldots, N-1$, is simply an empirical value between $0$ and $1$ for an appropriately chosen loss function $\ell$ that quantifies the lack-of-fit between the model.
\end{remark}

Note that the average (or mixture) loss strategies at each iteration time is given by  
\begin{align}
 L_n = \sum\nolimits_{k=1}^K \pi_n(k) r_n(k), \quad n = 0,1, 2, \ldots, N-1. \label{Eq2.8}
\end{align}
Then, our objective is to present an aggregation algorithm, that has a decision-theoretic minimization interpretation with dynamic allocation scenarios, which also guarantees an upper bound for the total overall mixture loss $L$, i.e.,
\begin{align}
 \min \to L &= \sum\nolimits_{n=0}^{N-1} L_n \notag \\
                 &= \sum\nolimits_{n=0}^{N-1}\sum\nolimits_{k=1}^K \pi_n(k) r_n(k). \label{Eq2.9}
\end{align}
Moreover, such an aggregation algorithm also ensures the following additional properties:
\begin{enumerate} [(i)]
 \item $r_n(k)$ tends to $0$ as $n \to \infty$ for all $k \in \{1,2, \ldots, K \}$.
\item $\bar{\Theta}_{N} = \sum\nolimits_{k=1}^K \pi_N(k) \Theta_{N}^{(k)}$ tends to the optimal parameter estimate $\theta^{\ast} \in \Gamma \subset \mathbb{R}^p$.
\end{enumerate}
In order to accomplish the above properties, we use a simple dynamic allocation strategy coupled with an iterative update scheme for computing the mixing distribution strategies $\pi_{n}(k)$, for each $k=1,2, \ldots, K$, i.e.,
\begin{align}
 \pi_{n}(k) = \frac{\omega_{n}(k)}{\sum\nolimits_{k=1}^K \omega_{n}(k)}, \label{Eq2.10}
\end{align}
while the weighting coefficients are updated according to
\begin{align}
 \omega_{n+1}(k) = \omega_{n}(k) \exp \left(r_{n}(k) \log(\beta)\right), \quad \beta \in (0,1),  \label{Eq2.11}
\end{align}
for $n = 0, 1, 2, \ldots, N-1$.

\begin{remark} \label{R2}
Note that the mixing distribution strategies $\pi_{n}(k)$, for $k=1,2, \ldots, K$, can be interpreted as a measure of quality corresponding to each of experts' current parameter estimates. Moreover, we can assign the initial weights $\omega_{0}(k)$, for $k=1,2, \ldots, K$, arbitrary values, but must be nonnegative and sum to one (see the Algorithm part in this section).
\end{remark}

Then, we state the following proposition that provides an upper bound for the total overall mixture loss.
\begin{proposition} \label{P1}
Let $L_n = \sum\nolimits_{k=1}^K \pi_n(k) r_n(k)$, $n = 0,1, 2, \ldots, N-1$, be a sequence of losses associated with the mixing distribution strategies $\pi_n =\left(\pi_n(1), \pi_n(2), \ldots, \pi_n(K)\right)$ and risk measures $r_n(k)$, for $k=1,2,\ldots, K$. Then, the total overall mixture loss $L$, i.e., $L = \sum\nolimits_{n=0}^{N-1}\sum\nolimits_{k=1}^K \pi_n(k) r_n(k)$, satisfies the following upper bound condition
\begin{align}
L \le -\frac{1}{1-\beta} \log \left(\sum\nolimits_{k=1}^K  \omega_{N}(k) \right).  \label{Eq2.12}
\end{align}
\end{proposition}
\begin{proof} Note that for any $\beta \in (0,1)$ and $r_n(k) \in [0, 1]$ for $k=1,2,\ldots, K$, we have the following inequality relation
\begin{align}
 \exp \left(r_{n}(k) \log(\beta)\right) \le 1 - (1- \beta)r_n(k), \label{Eq2.13}
\end{align}
due to the convexity argument for $\exp \left(r_{n}(k) \log(\beta)\right)\equiv\beta^{r_n(k)} \le 1 - (1 - \beta)r_n(k)$. Then, if we combine Equations~\eqref{Eq2.10} and \eqref{Eq2.11}, we will have
\begin{align}
\sum\nolimits_{k=1}^{K} \omega_{n+1}(k) &= \sum\nolimits_{k=1}^{K} \omega_{n}(k) \exp \left(r_{n}(k) \log(\beta)\right) \notag\\
                                                                   &\le \sum\nolimits_{k=1}^{K} \omega_{n}(k) \left(1 - (1- \beta)r_n(k) \right) \notag \\
                                                                   & =  \sum\nolimits_{k=1}^{K} \omega_{n}(k) -  (1- \beta) \sum\nolimits_{k=1}^{K} \omega_{n}(k)r_n(k) \notag \\
                                                                   & = \left(\sum\nolimits_{k=1}^{K} \omega_{n}(k)\right)  \left(1 - (1- \beta) \sum\nolimits_{k=1}^K \pi_n(k) r_n(k) \right). \label{Eq2.14}
\end{align}
Moreover, if we apply repeatedly for $n=0,1,2, \ldots, N-1$ to the above equation, then we have
\begin{align}
\sum\nolimits_{k=1}^{K} \omega_{N}(k) &\le \prod\nolimits_{n=0}^{N-1}\left(1 - (1- \beta) \sum\nolimits_{k=1}^K \pi_n(k) r_n(k) \right) \notag \\
                                                                &\le \exp \left( - (1- \beta) \sum\nolimits_{n=0}^{N-1}\sum\nolimits_{k=1}^K \pi_n(k) r_n(k) \right). \label{Eq2.15}
\end{align}
Notice that (due to the inequality relation $1 + t \le \exp(t)$ for all $t$) the right-hand side of the above equation satisfies the following inequality
\begin{align}
 1 - (1- \beta) \sum\nolimits_{n=0}^{N-1}\sum\nolimits_{k=1}^K \pi_n(k) r_n(k) \le \exp \left( - (1- \beta) \sum\nolimits_{n=0}^{N-1}\sum\nolimits_{k=1}^K \pi_n(k) r_n(k)\right), \label{Eq2.16}
\end{align}
that further gives us the following result
\begin{align}
\sum\nolimits_{k=1}^{K} \omega_{N}(k) \le \exp \left( - (1- \beta) \sum\nolimits_{n=0}^{N-1}\sum\nolimits_{k=1}^K \pi_n(k) r_n(k) \right). \label{Eq2.17}
\end{align}
Hence, the statement in the proposition follows immediately.
\end{proof}

Note that the dynamic allocation strategy in Equation~\eqref{Eq2.9} together with the iterative update scheme in Equation~\eqref{Eq2.9} provide conditions under which such an algorithm determines a set of mixing distribution strategies used for aggregating the experts' estimates that ultimately leading to an optimal parameter estimate, i.e., as a consensus solution for all experts, which is better than any individual expert's estimate in terms of improved generalization or learning performances.

Here, it is worth remarking that the risk measure $r_n(k)$ tends to $0$ as $n \to \infty$ for all $k \in \{1,2, \ldots, K \}$. Moreover, if we allow each expert to update its next parameter estimate with a modified initial condition $\bar{\Theta}_n= \sum\nolimits_{k=1}^K \pi_n(k) \Theta_{n}^{(k)}$ (i.e., the averaged value for the experts' current estimates) as follows
\begin{align}
 \Theta_{n+1}^{(k)} &= \bar{\Theta}_{n} - \delta\nabla J_k(\bar{\Theta}_{n},\hat{Z}^{(k)}) + \left(\epsilon/\sqrt{\log(\tau_{n+1} + 2)}\right)I_p \Delta W_n^{(k)}. \label{Eq2.18}
\end{align}
Then, $\bar{\Theta}_{N} = \sum\nolimits_{k=1}^K \pi_N(k) \Theta_{N}^{(k)}$ tends to the optimal parameter estimate $\theta^{\ast} \in \Gamma \subset \mathbb{R}^p$.

In what follows, we provide our algorithm that implements such an aggregation scheme for improving generalization in a typical class of learning problems.

{\rm \footnotesize

\begin{framed}

{\bf ALGORITHM:} Improving Generalization in a Class of Learning Problems
\begin{itemize}
\item[{\bf Input:}] The original dataset $Z^d = \bigl\{ (x_i, y_i)\bigr\}_{i=1}^d$; $K+1$ number of subsampled datasets; $m$ subsample data size. Then, by means of bootstrapping technique with/without replacement, generate $K+1$ subsample datasets:
\begin{align*}
\hat{Z}^{(k)} = \bigl\{ (\hat{x}_i^{(k)}, \hat{y}_i^{(k)})\bigr\}_{i=1}^m, \quad k=1,2, \ldots, K+1,
\end{align*}
with $(\hat{x}_i^{(k)}, \hat{y}_i^{(k)}) \in Z^d$ and $i \in \{1,\,2, \ldots, d\}$; an equidistant discretization time $\delta=\tau_{n+1} - \tau_n = T/N$, for $n=0,1,2, \ldots, N-1$, with $0=\tau_0 < \tau_1< \ldots < \tau_n<\ldots<\tau_N=T$, of the time interval $[0,T]$; $\gamma > 0$ and $\beta \in (0,1)$.
\item[{\bf 0.}] {\bf Initialize:} Start with $n=0$, and set $\pi_0(k) =\omega_0(k)= 1/K$, $\Theta_0^{(k)} = \theta_0$ for all $k=1,2,\ldots, K$.
\item[{\bf 1.}] {\bf Update Parameter Estimates:}
\begin{align*}
\bar{\Theta}_{n} &= \sum\nolimits_{k=1}^K \pi_{n}(k) \Theta_{n}^{(k)}\\
 \Theta_{n+1}^{(k)} &= \bar{\Theta}_{n} - \delta\nabla J_k(\bar{\Theta}_{n},\hat{Z}^{(k)}) + \left(\epsilon/\sqrt{\log(\tau_{n+1} + 2)}\right)I_p \Delta W_n^{(k)}, \quad k=1,2,\ldots, K\end{align*}
\item[{\bf 2.}] {\bf Update the allocation Distribution Strategies and the Weighting Coefficients:}
 \begin{itemize}
\item[{\bf i.}] {Compute the risk measure associate with each experts's estimate:}
\begin{align*}
r_n(k) = 1 - \exp\left(-\frac{\gamma}{m} \sum\nolimits_{i=1}^m {\ell}\left(h_{\Theta_{n}^{(k)}}(\hat{x}_i^{(K+1)}), \hat{y}_i^{(K+1)}\right) \right)
\end{align*}
 \item[{\bf ii.}] {Update the weighting coefficients:}
\begin{align*}
\omega_{n+1}(k) = \omega_{n}(k) \exp \left(r_{n}(k) \log(\beta)\right)
\end{align*}
 \item[{\bf iii.}] {Update the allocation distribution strategies:}
\begin{align*}
  \pi_{n+1}(k) = \frac{\omega_{n+1}(k)}{\sum\nolimits_{k=1}^K \omega_{n+1}(k)}
\end{align*}
 \end{itemize}
 for $k=1,2, \ldots, K$.
 \item[{\bf 3.}] Increment $n$ by $1$ and, then repeat Steps $1$ and $2$ until convergence, i.e., $\Vert \bar{\Theta}_{n+1} - \bar{\Theta}_{n}\Vert \le {\rm tol}$, or $n=N-1$.
 \item[{\bf Output:}] An optimal parameter value $\bar{\Theta}_{N} = \theta^{\ast}$.
\end{itemize}
\end{framed}}
Finally, it is worth mentioning that such a learning framework could be interesting to investigate from game-theoretic perspective with expert advice (e.g., see \cite{r4} and \cite{r5} for an interesting study in computer science literature).

 \section{Numerical simulation results} \label{S3} In this section, we presented numerical results for a simple nonlinear regression problem. In our simulation study, the numerical data for the population of {\it Paramecium caudatum}, which is a species of unicellular organisms, grown in a nutrient medium over $24$ days (including the starting day of the experiment), were digitized using the Software: WebPlotDigitizer \cite{r7} from the figures in the paper by F.G. Gause \cite{r5} (see also \cite{r6}). Here, our interest is to estimate parameter values for the population growth model, on the assumption that the model obeys the logistic law, i.e.,
 \begin{align*}
  N_{\theta}(t) = \frac{ N_0\,N_e}{N_0 + (N_e - N_0)\exp(-r t)}, \quad \theta = (N_0, N_e, r),
\end{align*}
 where $N_{\theta}(t)$ is the number of {\it Paramecium caudatum} population at time $t$ in $[{\rm Days}]$, and $N_0$, $N_e$ and $r$ (i.e., $\theta = (N_0, N_e, r)$) are the parameters to be estimated using the digitized dataset obtained from Gause's paper, i.e., the original dataset $Z^d = \bigl\{(t_i, N_i)\bigr\}_{i=1}^{d}$, where $d=23$ is the total digitized dataset points, $t_i$ is the time in $[{\rm Days}]$ and $N_i$ is the corresponding number of {\it Paramecium caudatum} population. Moreover, we generated a total of $26$ subsampled datasets of size $m = 23$ from the digitized original dataset by means of bootstrapping with replacement, i.e., the datasets $\hat{Z}^{(k)} = \bigl\{ (\hat{t}_i^{(k)}, \hat{N}_i^{(k)})\bigr\}_{i=1}^{23}$, with $k=1,2,\ldots, 25$, will be used for model training purpose, while the last dataset $\hat{Z}^{(26)} = \bigl\{ (\hat{t}_i^{(26)}, \hat{N}_i^{(26)})\bigr\}_{i=1}^{23}$ for generalization or model validation.
 \begin{center}
\begin{figure}[h]
\begin{center}
 \includegraphics[scale=0.235]{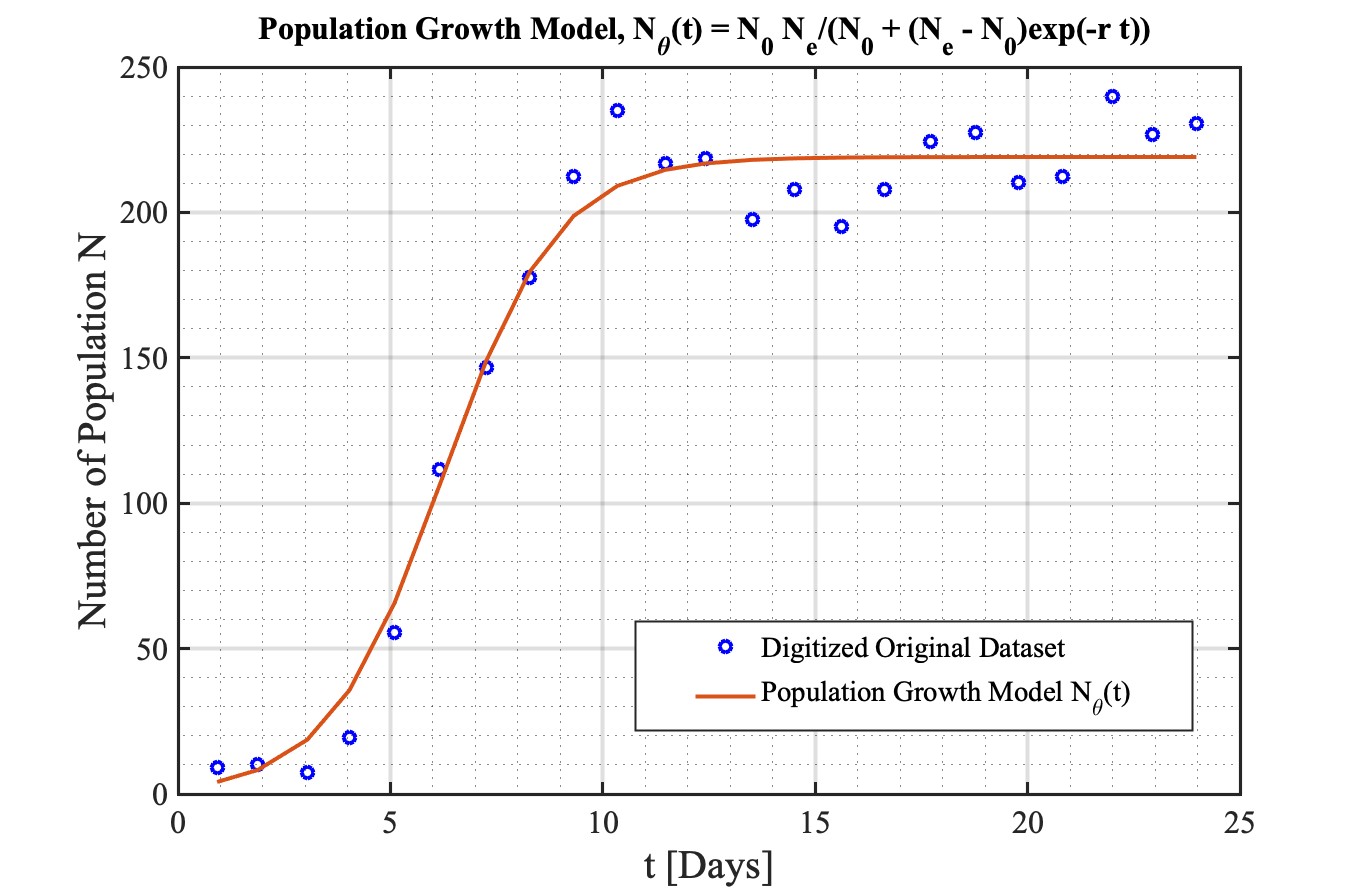}
  \caption{Plots for the original dataset and the population growth model.}\label{Fig1}
\end{center}
\end{figure}
\end{center} 
Note that we used a simple Euler--Maruyama time discretization approximation scheme to solve numerically the corresponding system of SDEs (cf. Equation~\eqref{Eq2.3}), with an equidistance discretization time $\delta = 1 \times 10^{-5}$ of the time interval $[0,1]$. For both model training and model validation processes, we used the usual quadratic loss function 
 \begin{align*}
 J_k(\theta,\hat{Z}^{(k)})=(1/23)\sum\nolimits_{i=1}^{23} \left(N_{\theta}(\hat{t}_i^{(k)}) - \hat{N}_i^{(k)}\right)^2, \quad k=1,2,\ldots, K+1,
\end{align*}
that quantifies the lack-of-fit between the model and the corresponding datasets. Figures~\ref{Fig1} shows both the digitized original dataset from Gause's paper and the population growth model $N(t)$, with global optimal parameter values $N_0^{\ast} = 2.1070$, $N_e^{\ast} = 219.0527$ and $r^{\ast} =0.7427$, versus time $t$ in $[{\rm Days}]$ on the same plot. In our simulation, we used a noise level of $\epsilon = 0.001$, and parameter values of $\gamma = 0.01$ and $\beta = 0.5$ (i.e., satisfying $\gamma > 0$ and $\beta \in (0,1)$, see the Algorithm in Section~\ref{S2}). Here, we remark that the proposed learning framework determined the parameter values for $N_0$ which is close to the initial experimental population size at time $t=0$, i.e., two {\it Paramecium caudatum} organisms. 
 
\section{Concluding remarks} \label{S4}
In this brief paper, we presented an algorithm that can be viewed as an extension for enhancing the learning performances in a typical empirical risk minimization-based learning problem, where the task of improving generalization is embedded in the learning process as a sequential decision-making problem with dynamic allocation scenarios. Moreover, we also provided conditions under which such an algorithm sequentially determines a set of distribution strategies used for aggregating across a group of experts' estimates that ultimately leading to an optimal parameter estimate, i.e., as a consensus solution for all experts, which is better than any individual expert's estimate in terms of improved generalization or performances. Finally, as part of this work, we presented some numerical results for a typical nonlinear regression problem.


\begin{thebibliography}{99}

\bibitem{r1}
{G.K. Befekadu}. Embedding generalization within the learning dynamics: An approach based-on sample path large deviation theory. {\em arXiv preprint}, arXiv:2408.02167, 2024.

\bibitem{r2}
{D.J. Higham}. An algorithmic introduction to numerical simulation of stochastic differential equations. {\em SIAM Rev.,} 43(3), 525--546. 2001. 

\bibitem{r3}
{E. Platen}. An introduction to numerical methods for stochastic differential equations. {\em Acta Numer.,} 8, 197--246, 1999. 

\bibitem{r4}
{V.G. Vovk}. A game of prediction with expert advice. {\em J. Comput. Syst. Sci.,} 56(2), 153--173, 1998.

\bibitem{r5}
{N. Littlestone \& M. Warmuth}. The weighted majority algorithm. {\em Inf. Comput.,} 108(2), 212--261, 1994.
 
\bibitem{r6}
{G.F. Gause}. Experimental analysis of Vito Volterra’s mathematical theory of the struggle for existence. {\em Science,} 79(2036), 16--17, 1934.

\bibitem{r7}
{G.F. Gause}. {\em The struggle for existence: A classic of mathematical biology and ecology}. United States: Dover Publications, 2019.

\bibitem{r8}
{A. Rohatgi}. Software: WebPlotDigitizer Version 5. May 14, 2024. Available at: \url{https://automeris.io}.

\end{thebibliography}
\end{document}